\documentclass{llncs}

\usepackage{graphicx}
\usepackage{lineno,hyperref}
\usepackage{amsmath}

\usepackage{algorithm}
\usepackage[algo2e]{algorithm2e} 
\usepackage{algorithmicx}
\usepackage{mathtools} \DeclarePairedDelimiter{\ceil}{\lceil}{\rceil}

\title{Uniform Scattering of Robots on Alternate Nodes of a Grid}

\author {Moumita Mondal$^1$, Sruti Gan Chaudhuri$^1$, Punyasha Chatterjee$^1$}
\institute{$^1$ Jadavpur University, Kolkata, India}

\begin{document}

\maketitle
\begin{abstract}
  In this paper, we propose a distributed algorithm to uniformly scatter the robots along a grid, with robots on alternate nodes of this grid distribution. These homogeneous, autonomous mobile robots place themselves equidistant apart on the grid, which can be required for guarding or covering a geographical area by the robots. The robots operate by executing cycles of the states "look-compute-move". In the look phase, it looks to see the position of the other robots; in the compute phase, it computes a destination to move to; and then in the move phase, it moves to that computed destination. They do not interact by message passing and can recollect neither the past actions nor the looked data from the previous cycle, i.e., oblivious. The robots are semi-synchronous, anonymous and have unlimited visibility. Eventually, the robots uniformly distribute themselves on alternate nodes of a grid, leaving the adjacent nodes of the grid vacant. The algorithm presented also assures no collision or deadlock among the robots.
\end{abstract}


\section{Introduction}
 A costly big complex robot may be replaced by a group of tiny autonomous robots also known as  {\it swarm robots} working collaboratively. To study on the collective behaviour of a group of tiny autonomous robots is an important area of research in the field of robotics. The swarm robots are basically mobile autonomous programmable particles. The goals of this autonomous mobile robot system have been patrolling, sensing, and exploring in a harsh environment like disaster-prone area, under the deep ocean and in the space with minimal human intervention \cite{Ref27}. The interest on distributed or decentralized control of multiple robots with constrained capabilities is a popular field of research for a long time. Theoretical representation of mobile robots in the Euclidean space have attracted researchers. The fundamental task for executing a job in collaboration is to form geometric shapes on the plane by the robots' positions. This pattern formation job has significantly supported in various fields, such as operations in hazardous environments, space missions, military operations, tumour excision, etc \cite{Ref7}. In this paper we address one of such pattern formation tasks, scattering in alternate nodes of a grid. The robots position themselves equidistant apart forming a grid in finite time. Each robot can sense its immediate surrounding or nearby robots, operates on the sensed data to compute the location to move to, and move to the computed destination. A significant application of the algorithm on scattering in grid, during Covid-19 pandemic, maybe in covering a region with autonomous robots, having UV ray emitting capabilities to disinfect the area \cite{Ref24} \cite{Ref25}. The robots freely move on a 2D plane. They are anonymous, homogeneous, oblivious and can not communicate but interact by observing others positions. Based on this model, we study the problem of scattering on alternate nodes of a grid by oblivious robots.

\subsection{Framework}
We consider a group of robots moving on a 2D plane. The robots are endowed with (a) Motorial capabilities i.e., they move independently in the Euclidean space and (b) Sensorial capabilities i.e., they look at the locations of the other robots. The robots have no means of communication through explicit message passing. However, the robots coordinate among themselves by observing the positions of the other robots on the plane. A robot is always able to look another robot within its visibility or sensing range (limited or unlimited). The robots are {\it homogeneous} (all executes the same algorithm) and {\it anonymous} (no unique identifiers). The {\it autonomy} of the robot system enables the robots to work without centralized control. The robots are assumed to be free of any kind computational and structural fault. The robots considered are point robots and are oblivious, i.e. they do not save  any information for the future.
The robots execute a cycle of the following phases:
\begin{itemize}
\item {\bf look} - Robots collect the positions of other robots within its sensing range.
\item {\bf Compute} - In the compute state, the robots execute our algorithm for computing  destinations. This algorithm is the same for all robots.
\item {\bf Move} - The robots move to the computed destinations in this state.
\end{itemize}
In general, the robots may or may not be synchronized.
\begin{itemize}
\item The {\bf asynchronous (ASYNC)} model is a pragmatic model, where the actions of the robots are totally independent. When a robot completes its computation, many of the other robots may have moved from their positions, based on which the computation is done. The asynchronous (ASYNC) scheduler activates the robots independently, and the duration of each Compute, Move and the time between successive activities is finite and unpredictable. As a result, a robot can be seen while moving and the snapshot and its actual configuration are not the same, and so its computation may be done with the old configuration.
\item {\bf semi-synchronous (SSYNC)} model, confirms that when a robot is moving no other robot is sensing. In this scheduling an arbitrary group of robots execute their cycles together and there is a common clock but, at each cycle, all robots may or may not be active. The semi-synchronous (SSYNC) scheduler activates a subset of all robots synchronously and their Look-Compute-Move cycles are performed at the same time. Here, we can assume that activated robots at the same time obtain the same snapshot, and their Compute and Move are executed instantaneously. In SSYNC, we can assume that each activation defines a discrete time called round and Look-Compute-Move is performed instantaneously in one round. A subset of activated robots in each round is determined by an adversary and robots do not know about this subset. \cite{Ref36}
\item In {\bf fully-synchronous (FSYNC)} model, all the robots operate their cycles at the same time. Thus all robots get the same view and they compute on the same data. It can be considered as a special case of SSYNC, when all robots are activated in each round.
\end{itemize} 
In this paper, we consider the semi-synchronous (SSYNC) model. 
We propose a robot's movement strategy such that, after a finite time they are placed equidistant apart on alternate nodes of a grid, leaving the adjacent nodes empty on the grid.

\section{Earlier Works and Our Contribution}

A large body of research work exist in the context of multiple autonomous mobile robots, exhibiting cooperative behaviour. The primary objective of such research work is to study the issues of group architecture, resource conflict, origin of cooperation, learning and geometric problems \cite{Ref19}. Traditionally the research on mobile robots involves artificial intelligence in which most of the results are based on experimental study or simulations. Recently, an emerging field of autonomous mobile robots looks at the robots as distributed mobile entities or programmable particle and investigate several coordination problems for them. They proceed to solve them deterministically providing proof of correctness of the algorithms.

The computational model for robots popular in the literature under this field is called a {\it weak model} \cite{Ref20}. Here, the robots execute repeated cycles consisting of phases, look-compute-move. The robots do not communicate through any wired or wireless medium. The robots may execute the cycle synchronously or semi-synchronously or asynchronously.

Scattering of mobile robots on a plane is one of the most popular problems. In this paper, the robot disperse on alternate nodes of a given infinite grid. The robots are initially on distinct arbitrary nodes of the grid and they can move along the edges of the grid to reach another node. The goal is to reach a state of static equilibrium in which they are uniformly scattered on alternate nodes of the grid. This uniform scattering (or covering, or self-deployment) problem occurs in practice, when robots are randomly deployed in a region but the requirement is for the region to be covered uniformly with maximizing coverage. A scattering algorithm specifies which operation must be performed by a robot whenever it is active, to achieve the given goal. All the robots perform the same self-deployment algorithm. Finally, the robots will reach a state of static equilibrium, and scattering will be completed within finite time. 
The subject of efficient uniform scattering is of extensive research in several fields (e.g. \cite{Ref2}, \cite{Ref3}, \cite{Ref5}, \cite{Ref11}, \cite{Ref18}, \cite{Ref7},\cite{Ref8}, \cite{Ref9}, \cite{Ref10}). Depending on the assumptions they make, the existing protocols differ greatly from each other.
Some of the major differences are based on the following:
\begin{itemize}
\item Nature of the environment – The robots can move on  a Euclidian plane, called continuous (e.g., \cite{Ref5}, \cite{Ref18}), or in a network or a graph, usually called discrete or graph world (e.g., \cite{Ref1}, \cite{Ref9}, \cite{Ref11}).

\item The robots may (e.g., \cite{Ref2}, \cite{Ref11}, \cite{Ref18}) or may not be synchronized (e.g., \cite{Ref5})

\item Memory of the robots- The robots can have a persistent memory (e.g., \cite{Ref2}, \cite{Ref7}, \cite{Ref18}) or oblivious (e.g., \cite{Ref5}, \cite{Ref10});

\item Visibility or sensing range of the robots may be limited (e.g., \cite{Ref7}, \cite{Ref11}) or unlimited, i.e., extends to the entire region (e.g., \cite{Ref5});

\item Computational power - the robots  have  the computational power of  Turing machines (e.g., \cite{Ref12}, \cite{Ref18}) or are simple Finite State machines (e.g., \cite{Ref1}, \cite{Ref11});

\item  Nature of termination of protocol - exact or approximate uniform covering is reached within finite time (e.g., \cite{Ref5}, \cite{Ref11}) or the protocol converges without ever terminating (e.g., \cite{Ref2}, \cite{Ref3});

\item Type of protocol (generic/specific) - the protocol is generic, i.e., it operates in any space/network (e.g., \cite{Ref5}, \cite{Ref12}, \cite{Ref18}) or only in specific for classes of regions/graphs (e.g., \cite{Ref2}, \cite{Ref3})

\end{itemize}

 Barriere et al., \cite{Ref22} discussed the uniform scattering problem for a set of autonomous mobile robots deployed in a grid network. In this paper, robots with a constant memory (non-oblivious) and limited visibility are considered. The dispersion problem of mobile robots on graphs is also discussed by Ajay D. Kshemkalyani et al., \cite{Ref23} where the robots are initially placed arbitrarily on the nodes of an n-node anonymous graph and they autonomously reposition themselves to reach a configuration in which each robot is on a distinct node of the graph. Also in \cite{Ref29} \cite{Ref30} \cite{Ref31} \cite{Ref34} the dispersion problem is discussed.

\section{Our Contribution}

In this paper, we prove that starting from any arbitrary initial configuration of robots on a given infinite grid, the robots can uniformly distribute themselves along a grid, with robots on the alternate nodes of this grid distribution. This problem is of significantly important due to its relationship to many other fundamental robot coordination problems, such as exploration, scattering, load balancing, relocation of self-driven electric cars (robots) to recharge stations (nodes), etc. The protocol is fully localized and decentralized, and it makes minimal assumptions like; it does not require any direct or explicit communication between robots; the robots have no past memory; the robots are anonymous, semi-synchronous and identical. We propose a distributed algorithm to disperse unlimited visibility robots in alternate nodes of a grid. We show that if the robots agree only on the direction of both axes (i.e., X-axis and Y-axis), then they can form a uniform distribution along a grid, with alternate nodes empty, without encountering any collision or deadlock.

\section{Algorithm}
A set of $n$ stationary points on an infinite grid on a 2D plane is given. The set of robots, R is scattered on various nodes of the grid. The robots move along the edges of the grid, from one grid node to another in such a way that after a finite number of cycle execution they are placed uniformly, equidistantly apart on alternate nodes of a grid.

\subsection{Underlying Model}
Let R = $\{r_{1}, r_{2},..,r_{n}\}$ be a set of autonomous mobile robots. The set of robots R deployed on the given grid is described as follows:
\begin{itemize}
    \item The robots are autonomous.
    \item Robots are anonymous and homogeneous.
    \item The robots are oblivious in the sense that they can not recollect any data from the past cycle.
    \item Robots can not communicate explicitly. Each robot is allowed to have   a camera that can take pictures over 360 degrees. The robots communicate only by means of observing other robots with the camera.
    \item The robots are point robots.
     \item The robots have unlimited visibility.
     \item The robots execute look-compute-move under semi-synchronous schedules. In this scheduling, a set of robots execute the cycles synchronously. However, this set is chosen randomly. This scheduling gives the assurance that a robot will not capture the locations of moving robots.
     \item A robot considers its position as origin (i.e its local co-ordinate system). The robots do not have any global origin. However, they agree on the direction and orientation of X and Y axes.
     \item The robots reside in nodes and do not stop in edges. During the movement to the computed destination node, they do not stop in between nodes. This motion is also know is rigid motion.
\end{itemize}

\subsection{Overview of the Problem}
Our objective is to uniformly distribute the autonomous mobile robots, along a grid, such that the robots are on alternate nodes of this grid distribution. Initially all the robots are on distinct nodes. 
Following are the steps to be executed by each robot in the compute phase: 
\begin{itemize}
\item Determine the maximum number of rows and columns required to uniformly distribute the robots along a grid such that the robots are on alternate nodes of this grid distribution, using the {\it FindDimension} routine;
\item Compute the north-most bound, (using the {\it FindYMAX} routine) and west-most bound (using the {\it FindXMIN} routine) of the uniform robot distribution along the grid;
\item The robots uniformly distribute themselves along the grid, leaving the alternate nodes empty, using the {\it FormGrid} routine.
\end{itemize}

\subsection{Description of Algorithm FormGrid}
Let the number of the robots in R be $n$ and a robot, $r_{i} \in R$. A grid node, $T_{x,y}$ has coordinate $(x,y)$. Let $grid_{final}$ be the uniform distribution of robots to be formed along the grid, with robots on alternate nodes of this grid distribution and $row_{j}$ be the $j^{th}$ row of this $grid_{final}$. So, the number of robots, in the alternate rows or columns of $grid_{final}$, will be $\ceil*{\sqrt{n}}$ (say, $rc$), such that the last row(s) of the grid may be completely or partially empty. The maximum dimensions of the final uniform distribution, $grid_{final}$, with robots on alternate nodes, is determined as $((rc*2)-1)$ (say, $d$), since the alternate rows and columns will be empty (Algorithm ~\ref{algo1}).

The Y-axis value of the extreme north robot ($Y_{MAX}$) is the northwards bound of $grid_{final}$ and this row is also considered the first row, $row_{1}$ (i.e. j=1) of the uniform distribution of robots. To determine the $Y_{MAX}$, a robot $r{i}$ compares the Y-axis values of all the robots. So, the robot with the maximum Y-axis value is the north-most robot and its Y-axis value is considered as the $Y_{MAX}$ (Algorithm ~\ref{algo2}). Similarly, the X-axis value of the extreme west robot ($X_{MIN}$) is the westwards bound of the $grid_{final}$. To determine the $X_{MIN}$ also, a robot $r{i}$ compares the X-axis values of all the robots. The robot with the minimum X-axis value is the west-most robot and its X-axis value is considered as the $X_{MIN}$ (Algorithm ~\ref{algo3}). Although the robots locally compute these values, they all agree to the same robot as north-most and thus have the same $Y_{MAX}$, as they agree on the direction and orientation of the axes. Similarly, the robots also agree to the same robot as west-most and thus have the same $X_{MIN}$.

The robots, $r_{i}$ compute their movement from their source node, $T_{x,y}$ to a destination node on the grid. In case of any tie, the robot trying to move westwards to its destination node will get highest priority, followed by eastwards movement and then southwards and northwards respectively.
Depending on the position of $r_{i}$ on the grid, there are 5 configurations, denoted as $\Psi_{1}$, $\Psi_{2}$, $\Psi_{3}$, $\Psi_{4}$ and $\Psi_{5}$.

\begin{itemize}
\item CASE $\Psi_{1}$: Robot, $r_{i}$ is in a row, $row_{j}$ such that it is among the first $d$ rows from $Y_{MAX}$ (i.e. $j<=d$) and it is an even row, (i.e. j is even).
In this case, $r_{i}$ tries to move northwards to an odd row. $r_{i}$ checks if $T_{x,y+1}$ is vacant and then moves there. Otherwise, if $T_{x,y+1}$ is not empty, $r_{i}$ moves a node eastwards $T_{x+1,y}$ and then checks if the immediate north node is vacant. If $T_{x+1,y}$ is also not vacant, $r_{i}$ waits for it to be vacant. This process continues until $r_{i}$ reaches the odd row northwards to it, i.e $row_{j-1}$.
Thus the robots, on the even rows, $row_{j}$ (where, j is even), of the first $d$ rows from the $Y_{MAX}$, move to the odd rows ($row_{j-1}$) above it. This continues till all $r_{i}$ of the first $d$ rows from the $Y_{MAX}$, are placed on a odd row. (Figure ~\ref{1}(i)) \\

\item CASE $\Psi_{2}$: Robot, $r_{i}$ is in a row, $row_{j}$ such that it is beyond the first $d$ rows from $Y_{MAX}$ (i.e. $j>d$).
In this case, $r_{i}$ tries to move northwards, aiming to reach the $d^{th}$ row of the $grid_{final}$. $r_{i}$ checks if immediate above node, $T_{x,y+1}$ is vacant and then moves there. Otherwise, if $T_{x,y+1}$ is not empty, $r_{i}$ moves eastwards to a vacant $T_{x+1,y}$ and then to the row northwards, $T_{x+1,y+1}$. If $T_{x+1,y}$ is also not vacant, $r_{i}$ waits for it to be vacant. This process continues until $r_{i}$ reaches the $d^{th}$ row from $Y_{MAX}$, i.e $row_{d}$. (Figure ~\ref{1}(ii))\\

\item CASE $\Psi_{3}$: All the $n$ robots are in alternate rows (i.e. odd rows) of the first $d$ rows starting from $Y_{MAX}$. Robot, $r_{i}$ placed in $row_{j}$, is not on alternate nodes of the $row_{j}$. Also, $r_{i}$ is not on the west bound of $grid_{final}$, $X_{MIN}$. Here, there can be two possibilities:
\begin{itemize}
\item Any robot on the west of $r_{i}$ in $row_{j}$, have a vacant target node to move westwards.
$r_{i}$ tries to move westwards in its row, till the westwards bound, $X_{min}$ is meet or there is no further vacant target node to move westwards. $r_{i}$ checks if westwards nodes, $T_{x-1,y}$ and $T_{x-2,y}$ are vacant, if so then $r_{i}$ moves westwards to $T_{x-1,y}$.
\item All robots on the west of $r_{i}$ in $row_{j}$, have no vacant target node to move westwards.
Now, $r_{i}$ tries to move eastwards, to position itself in alternate nodes of the row. If checks if the eastwards nodes, $T_{x+1,y}$ and $T_{x+2,y}$ are vacant, if so then $r_{i}$ moves eastwards to $T_{x+1,y}$, else it waits. 
\end{itemize}
This process continues until all the robots in $row_{j}$ are on alternate nodes.\\
Following this, all the robots are on alternate rows (i.e., odd rows only) and also on alternate nodes of the same row. (Figure ~\ref{2}, Figure ~\ref{3}, Figure ~\ref{4}(ii), Figure ~\ref{5}). However, after all the robots are placed on the alternate nodes, the rows may have greater or less than the $rc$ number of robots.\\

\item CASE $\Psi_{4}$: All robots are placed in alternate nodes of alternate rows. Robot, $r_{i}$ in $row_{j}$ is such that starting from the west bound, $X_{MIX}$, $r_{i}$ is in the first $rc$ number of robots in $row_{j}$, and the northwards alternate row (odd row) node, $T_{x,y+2}$ is empty. Also, $row_{j}$ is not the north-most row (i.e., $j \neq 1$).
In this case, $r_{i}$ moves to the northwards alternate row, $T_{x,y+2}$.(Figure ~\ref{4}(i))\\
 
\item CASE $\Psi_{5}$: Robot, $r_{i}$ has greater than or equal to $rc$ number of robots on its west in the same row. Here, there can be two possibilities:
\begin{itemize}
\item $r_{i}$ is in the north-most row, $row_{1}$. Here, $r_{i}$ tries to move to a southwards odd row. $r_{i}$ checks if its immediate alternate south row node, $T_{x,y-2}$ is vacant, if so, then moves southwards to node $T_{x,y-2}$, else waits. After $r_{i}$ moves to node $T_{x,y-2}$ in $row_{j+2}$, it then tries to move westwards, in case $row_{j+2}$ has less than $rc$ robots, as per, $\Psi_{3}$. Otherwise, $r_{i}$ continues to move to a further south odd row and when it encounters a $row_{j}$ with lesser then $rc$ robots, $r_{i}$ moves westwards to place itself in its target node as per, $\Psi_{3}$.
\item $r_{i}$ is not the north-most row (i.e., $j>1$). Here, the robots try to move to a northward or southward odd row to form the uniform distribution, $grid_{final}$. The robots that has $rc$ number of robots on their west on the same row, are the excess robots on that row. So, they move to the other odd rows that has lesser than $rc$ number of robots.
In this case, if odd rows above $row_{j}$ have less than $rc$ robots, and the northwards alternate row node, $T_{x,y+2}$ is vacant, then $r_{i}$ moves northwards to the alternate row node, $T_{x,y+2}$.
Then if $r_{i}$ has less than $rc$ number of robots on its west in $row_{j-2}$, it tries to place itself in that row, as per $\Psi_{3}$.
If $T_{x,y+2}$ is not vacant, $r_{i}$ waits till $T_{x,y+2}$ is empty.
Only when all odd rows above $row_{j}$ have greater than or equal to $rc$ robots and the southwards alternate row node, $T_{x,y-2}$ is vacant, then $r_{i}$ moves southwards to the alternate row node, $T_{x,y-2}$. 
Then if $r_{i}$ has less than $rc$ number of robots on its west in $row_{j+2}$, it tries to place itself in that row, as per $\Psi_{3}$. 
If $T_{x,y-2}$ is not vacant, $r_{i}$ waits till $T_{x,y-2}$ is empty.
\end{itemize}
This process continues until no row has greater than $rc$ robots. (Figure ~\ref{6}(i)). 
\end{itemize}
Thus the robots uniformly distribute themselves along a grid with robots on alternate nodes of this grid distribution, by computing their target nodes on the grid (Algorithm ~\ref{algo4}) (Figure ~\ref{6}(ii)). An example for the algorithm, with 8 arbitrarily placed robots on the nodes of an infinite grid is shown in Figures 1,2,3,4,5,6 respectively.

 \begin{figure}[H]
   \centering
   \includegraphics[scale=0.52]{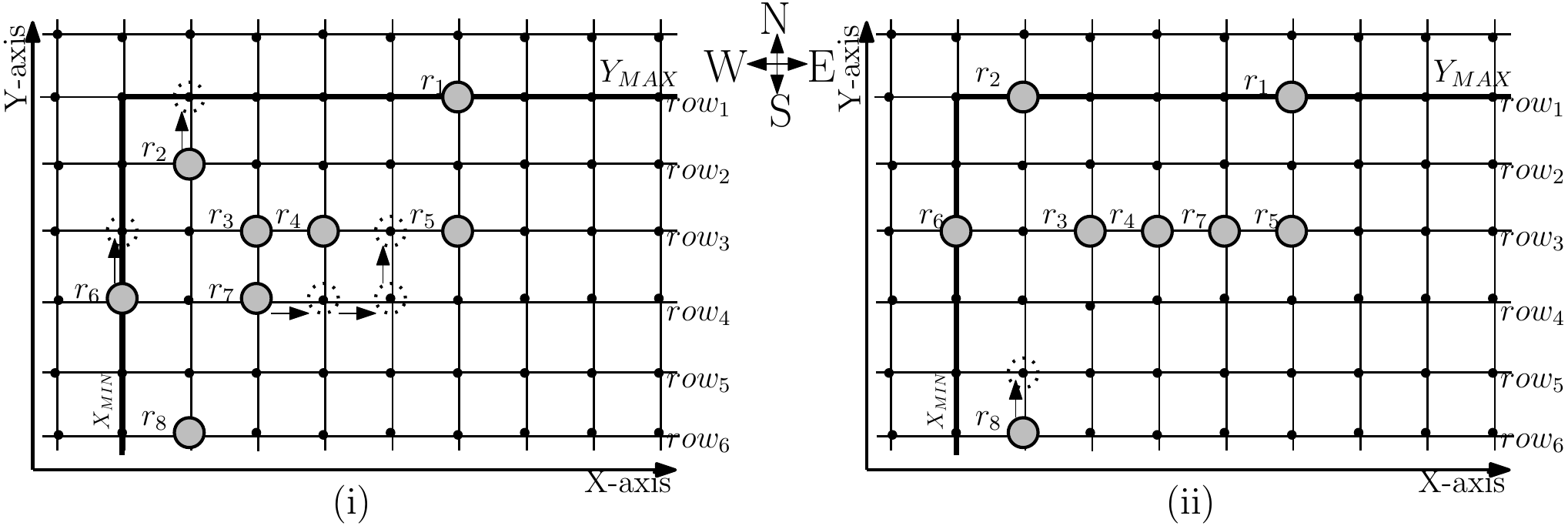}
   \caption{(i) The robots on the even rows move northwards to the odd rows of $grid_{final}$.(CASE $\Psi_{1}$)\\
   (ii) $r_{8}$ is beyond the first $d$ (i.e (3*2)-1=5) rows from $Y_{MAX}$, so it moves northwards to the $5^{th}$ row.(CASE $\Psi_{2}$)
  }
   \label{1}
  \end{figure}

\begin{figure}[H]
   \centering
   \includegraphics[scale=0.52]{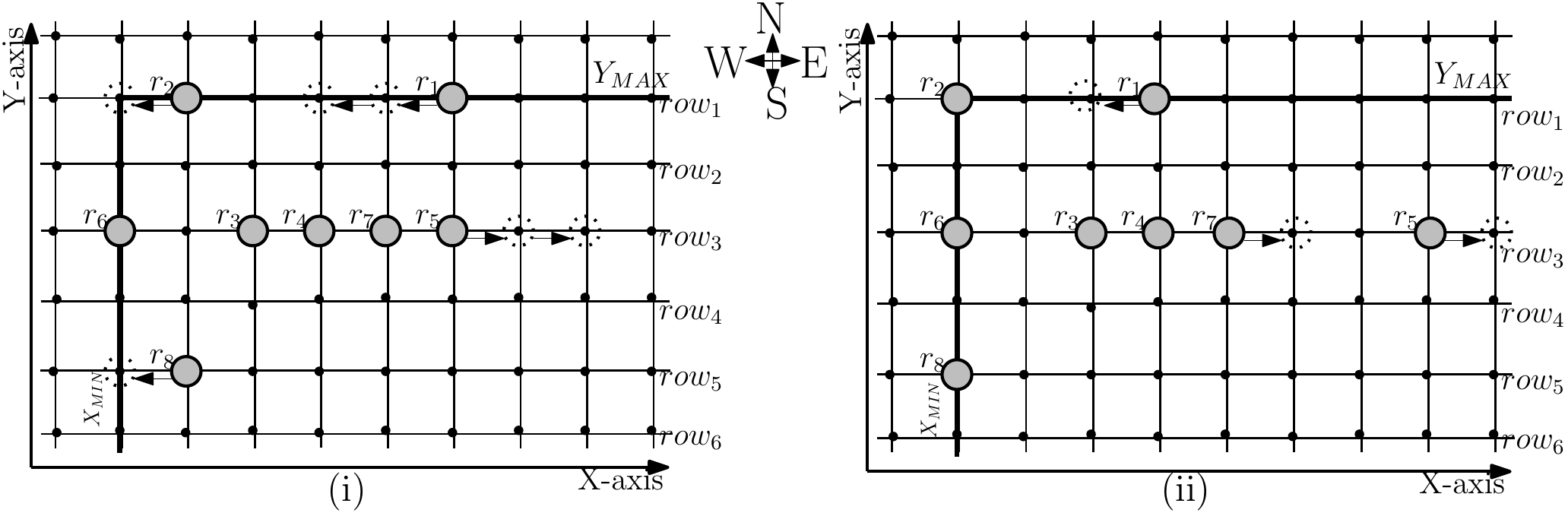}
   \caption{(i) and (ii)The robots move to their immediate west node on the same row when the two consecutive west nodes are vacant. Robots try to move eastwards only when there is no vacant target node to move west.(CASE $\Psi_{3}$)}
   \label{2}
  \end{figure}

  \begin{figure}[H]
   \centering
   \includegraphics[scale=0.52]{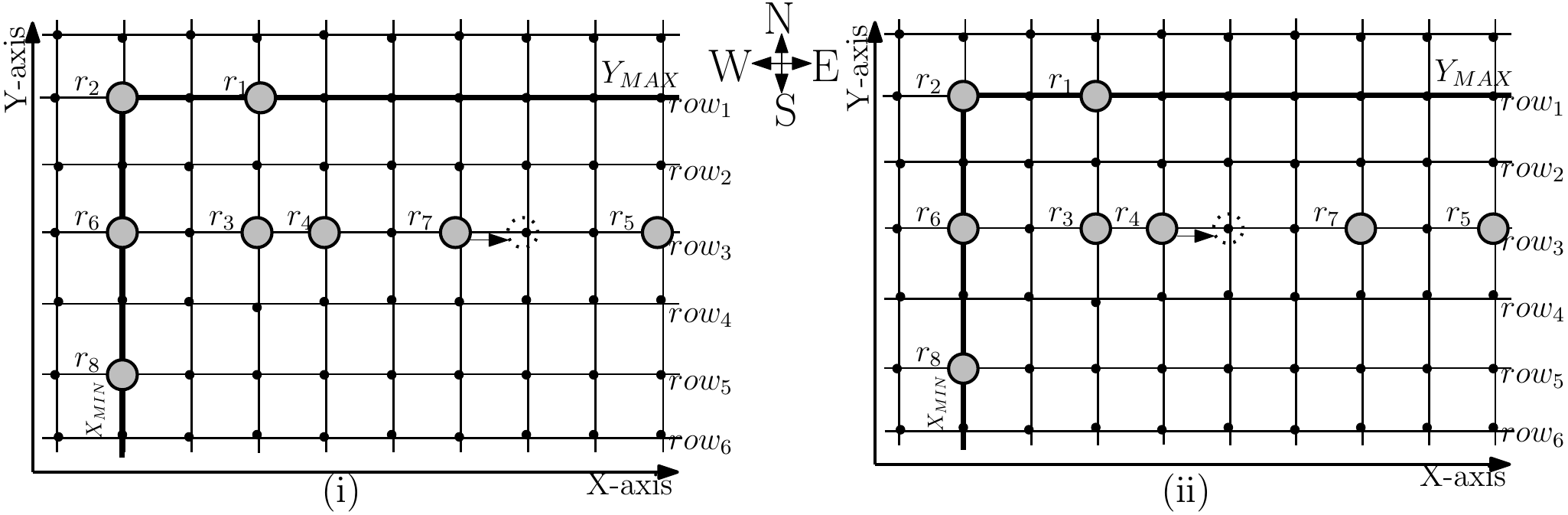}
   \caption{(i) and (ii) The robots move further to place themselves on alternate nodes of the same row.(CASE $\Psi_{3}$)}
   \label{3}
  \end{figure}

 \begin{figure}[H]
   \centering
   \includegraphics[scale=0.52]{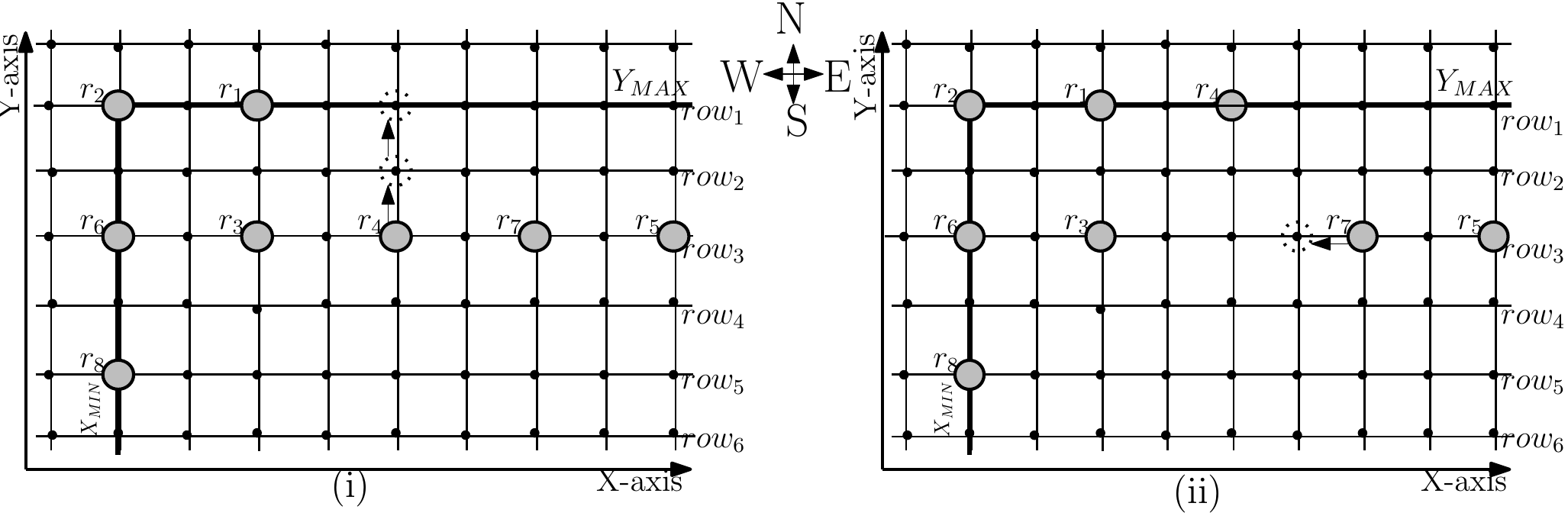}
   \caption{(i) Since all robots on the alternate rows are placed on the alternate nodes, $r_{4} \in$ first $rc$ robots of $row_{3}$, moves northwards to the alternate row, $row_{1}$, as it finds a vacant node.(CASE $\Psi_{4}$)\\
   (ii) Robot $r_{7}$ in $row_{3}$ again continue to hop west, since it again finds consecutive vacant nodes on its west.(CASE $\Psi_{3}$)}
   \label{4}
  \end{figure}

  \begin{figure}[H]
   \centering
   \includegraphics[scale=0.52]{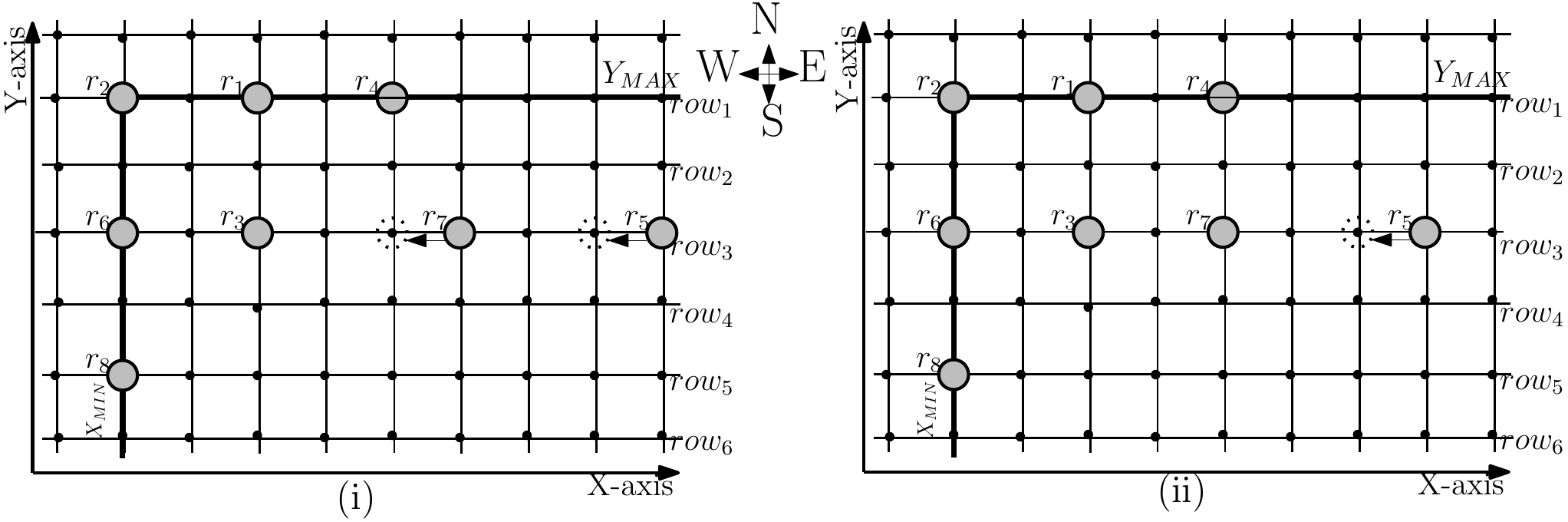}
   \caption{(i) Robot $r_{7}$ and then $r_{5}$ in $row_{3}$ continue to hop west, as they find two consecutive vacant nodes on their west.(CASE $\Psi_{3}$)\\
   (ii) Robot $r_{5}$ continue to move west to place itself on an alternate node of that row.(CASE $\Psi_{3}$)}
   \label{5}
  \end{figure}

\begin{figure}[H]
   \centering
   \includegraphics[scale=0.52]{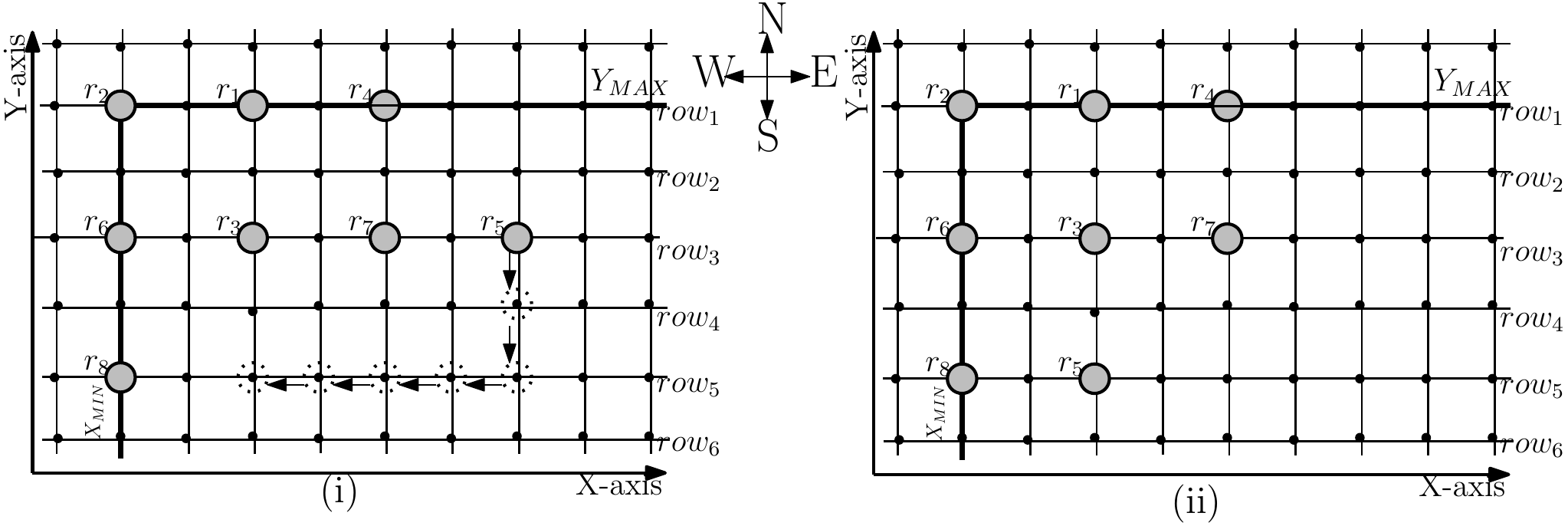}
   \caption{(i) Robot $r_{5}$ in $row_{3}$ has $rc$ robots on its west in $row_{3}$ and the above odd row $row_{1}$ has $rc$ robots.(CASE $\Psi_{5}$). So, $r_{5}$ moves southwards to alternate row, $row_{5}$ and then move westwards to place itself on an alternate node of the grid distribution.\\
   (ii)An uniform distribution along the grid ($grid_{final}$) is formed, with robots on alternate nodes of this grid distribution.}
   \label{6}
  \end{figure}

\begin{algorithm} [H]
\caption{FindDimension($n$)}
\label{algo1}
\KwIn{$n$}
\KwOut{Maximum dimensions of the uniform distribution of robots, to be formed} 
Let $grid_{final}$ be uniform distribution of robots to be formed, with robots on alternate nodes of this grid distribution;\\ $rc$ be the maximum number of robots, in a row or column of $grid_{final}$; and \\
$d$ be the maximum number of rows or columns of $grid_{final}$;\\
$rc \leftarrow \ceil*{\sqrt{n}}$; \\
$d \leftarrow (rc*2)-1$;\\
\Return $rc,d$
\end{algorithm}

\begin{algorithm} [H]
\caption{FindYMAX($n$)}
\label{algo2}
\KwIn{$n$}
\KwOut{The north-most bound of the distribution, $Y_{MAX}$}
$r{i} \in R$ considers its position as origin (0,0). \\
Let $Y_{MAX}$ be the Y-axis value of the north-most robot in the grid and \\ 
Initially let, $Y_{MAX}\leftarrow$ Y-axis value of $r_{1}$; \\ 
Let c be the robot counter to compare all robots and $c\leftarrow2$;\\
\While {$c<=n$}
{
    \If {Y-axis value of $r_{c} > Y_{MAX}$}
    {
        $Y_{MAX}\leftarrow$ Y-axis value of $r_{c}$;
    }
    $c\leftarrow c+1$;
}
\Return $Y_{MAX}$
\end{algorithm}

\begin{algorithm} [H]
\caption{FindXMIN($n$)}
\label{algo3}
\KwIn{$n$}
\KwOut{The west-most bound of the distribution, $X_{MIN}$}
$r{i} \in R$ considers its position as origin (0,0). \\
Let $X_{MIN}$ be the X-axis value of the west-most robot and \\
Initially let, $X_{MIN}\leftarrow$ X-axis value of $r_{1}$; \\ 
Let c be the robot counter to compare all robots and $c\leftarrow2$;\\
\While {$c<=n$}
{
    \If {X-axis value of $r_{c} < X_{MIN}$}
    {
        $X_{MIN}\leftarrow$ X-axis value of $r_{c}$;
    }
    $c\leftarrow c+1$;
}
\Return $X_{MIN}$
\end{algorithm}

\begin{algorithm} [H]
\caption{FormGrid($n$)}
\label{algo4}
\KwIn{$n$}
\KwOut{Robot $R$ reaches its target point on the grid}
Let $r_{i} \in$ R and $T(x,y)$ be a grid node with coordinate $(x,y)$; \\
Initially, consider $r_{i}$ is at $T(x,y)$ (where x,y=0). Let the final target point of $r_{i}$ be $T_{i}$\\
Dimensions of the uniform distribution to be formed ($grid_{final}$) $\leftarrow$ {\it FindDimendion};\\
$Y_{MAX} \leftarrow$ {\it FindYMAX}; and
$X_{MIN} \leftarrow$ {\it FindXMIN};\\
$j$ is the row number of $row_{j}$ starting from the north bound, $Y_{MAX}$\\
Priority of movement in case of any tie in $r_{i}$ movement $\leftarrow$ $r_{i}$ trying to move westwards will get highest priority, followed by eastwards, southwards and northwards respectively;

\If {$r_{i} \in row_{j}$ such that ($j<=d$ and $j \pmod 2 = 0$)}
{
    \While {$r_{i} \in row_{j}$}
    {
        \If {$T_{x,y+1}$ is empty} 
        {
            $r_{i}$ moves to $T_{x,y+1}$; \\
            j=j-1;
        }
        \ElseIf {$T_{x+1,y}$ is empty}
        {
           $r_{i}$ moves to $T_{x+1,y}$;
        }
    }
}

\While {$r_{i} \in row_{j}$ such that $j>d$}
{
    \If {$T_{x,y+1}$ is empty}
    {
        $r_{i}$ moves to $T_{x,y+1}$;  \\
        j=j-1;
    }
    \ElseIf {$T_{x+1,y}$ is empty}
    {
        $r_{i}$ moves to $T_{x+1,y}$;
    }
   
}
\end{algorithm}

\begin{algorithm}[H]         
\caption {FormGrid($n$) continued}

\While {$r_{i}$ $\notin$ $X_{MIN}$ and $\forall$ $r_{i}$ $\in$ $row_{j}$ not on alternate nodes}
{
    \If {$T_{x-1,y}$ and $T_{x-2,y}$ are empty}
    {
        $r_{i}$ moves to $T_{x-1,y}$;
    }
    \ElseIf {all robots on the west of $r_{i}$ in $row_{j}$, have no vacant target node to move westwards}
    {
        \If {$T_{x+1,y}$ and $T_{x+2,y}$ are empty}
        {
            $r_{i}$ moves to $T_{x+1,y}$;
        }
    }
    
}


\While {$\forall r_{i}$ placed in alternate nodes of alternate rows and $r_{i} \in row_{j}$ such that $r_{i} \in$ first $rc$ robots of $row_{j}$ and $j \neq 1$ and $T_{x,y+2}$ is empty}
{
    $r_{i}$ moves to $T_{x,y+2}$;\\
    j=j-2;
}
\If {$r_{i} \in row_{j}$ such that, $j=1$ and $r_{i}$ has $>= rc$ number of robots on its west in $row_{j}$}
{
        \While {$T_{x,y-2}$ is empty and $r_{i}$ $\notin$ its target point $T_{i}$}
        {    
            $r_{i}$ moves to $T_{x,y-2}$;  \\
            j=j+2;\\
            \If {$r_{i}$ has $< rc$ number of robots on its west in $row_{j}$}
            {      
                \While {($T_{x-1,y}$ and $T_{x-2,y}$ are empty) and $r_{i}$ not on $X_{MIN}$}
                {
                    $r_{i}$ moves to $T_{x-1,y}$;
                }
                $r_{i}$ placed in target point $T_{i}$;
            }
        }
       
}

\If {$r_{i} \in row_{j}$ such that, $j>1$ and $r_{i}$ has $>= rc$ robots on its west in $row_{j}$}
{
        \If {odd rows above $row_{j}$ have $< rc$ robots}
        {
            \While {$T_{x,y+2}$ is empty and $r_{i}$ $\notin$ its target point $T_{i}$}
            {
                $r_{i}$ moves to $T_{x,y+2}$; \\
                j=j-2;\\
                \If {$r_{i}$ has $< rc$ number of robots on its west in $row_{j}$}
                {
                    \While {($T_{x-1,y}$ and $T_{x-2,y}$ are empty) and $r_{i}$ not on $X_{MIN}$}
                    {  
                        $r_{i}$ moves to $T_{x-1,y}$;
                    }
                    $r_{i}$ placed in its target point $T_{i}$;
                }
            }
            
        }
        \ElseIf {odd rows above $row_{j}$ have $>= rc$ robots}
        {
            \While {$T_{x,y-2}$ is empty and $r_{i}$ $\notin$ its target point $T_{i}$}
            {
                $r_{i}$ moves to $T_{x,y-2}$; \\
                j=j+2;\\
                \If {$r_{i}$ has $< rc$ robots on its west in $row_{j}$}
                {
                    \While {($T_{x-1,y}$ and $T_{x-2,y}$ are empty) and $r_{i}$ not on $X_{MIN}$}
                    {
                        $r_{i}$ moves to $T_{x-1,y}$;
                    }
                    $r_{i}$ placed in its target point $T_{i}$;
                }
            }
        }
}
\end{algorithm}

\section{Correctness}

\begin{lemma}
The robots do not move back to their previous positions.
\end{lemma}
\begin{proof}
 The robots reside in nodes and do not stop in edges. Also, during the movement to the computed destination node, they do not stop in any in-between nodes, i.e. they execute rigid motion. So, in configurations $\Psi_{4}$ and $\Psi_{5}$, a robot $r_{i}$ in $row_{j}$ can compute and move to the next odd row, northwards or southwards (i.e. $row_{j+2}$ or $row_{j-2}$). Thus, in Figure ~\ref{4}(i), $r_{i}$ computes and moves to $T_{x,y+2}$ and in Figure ~\ref{6}(i), $r_{i}$ moves to $T_{x,y-2}$. As $r_{i}$ do not compute and move to its immediate adjacent row (i.e. even row, $row_{j+1}$ or $row_{j-1}$) and thus not reaching the configuration $\Psi_{1}$, so by the algorithm, $r_{i}$ does not move back to its previous position in $row_{j}$. Consequently, the robots move forward in the algorithm, to achieve the required configuration, without repeating the same positions.
\end{proof}

\begin{lemma}
The algorithm is free of deadlock. Hence progress is assured.
\end{lemma}
\begin{proof}  
Let $r_{i}$ be any arbitrary robot. By contradiction, let us assume $r_{i}$ is in deadlock.
\begin{itemize}
    \item Case 1: If $r_{i}$ is in $\Psi_{1}$, it checks if northward node $T_{x,y+1}$ is vacant and moves there. Otherwise, $r_{i}$ first moves eastwards to $T_{x+1,y}$ and then northwards. Thus, we can see, $r_{7}$ in Figure ~\ref{1}(i), that in $\Psi_{1}$, $r_{i}$ follows distinct movement either eastward or northward as per the condition in Algorithm ~\ref{algo4} and finally moves to an odd row. $r_{i}$ do not repeat the actions of $\Psi_{1}$ further, as it no longer remain in $\Psi_{1}$. This assures that $r_{i}$ is not in deadlock.
    
    \item Case 2: When $r_{i}$ is in $\Psi_{2}$, it checks if northward node $T_{x,y+1}$ is vacant and then moves there. Otherwise, $r_{i}$ first moves eastwards to $T_{x+1,y}$ and then northwards. $r_{i}$ waits if both $T_{x,y+1}$ and $T_{x+1,y}$ are occupied. Thus similar to $\Psi_{1}$, $r_{i}$ moves either either northward or eastward till $r_{i}$ reaches the $d^{th}$ row of the $grid_{final}$, as $r_{8}$ in Figure ~\ref{1}(ii). Further it does not repeat the $\Psi_{2}$ actions, as it no longer remains in $\Psi_{2}$, thus confirming there is no repetition of actions in loop. Hence, $r_{i}$ cannot be in deadlock.
    
    \item Case 3: When $r_{i}$ is in $\Psi_{3}$, it moves westwards to $T_{x-1,y}$, when two adjacent west nodes, $T_{x-1,y}$ and $T_{x-2,y}$ are vacant as $r_{1}$, $r_{2}$, $r_{8}$ of Figure ~\ref{2}(i). Otherwise, $r_{i}$ moves eastwards as $r_{5}$ of Figure ~\ref{2}(i). If both west and eastwards nodes are unavailable for movement, it waits. Also, in Figures ~\ref{2}(ii), ~\ref{3}(i)(ii), ~\ref{4}(ii), ~\ref{5}(i)(ii) we see, there is always either westward or eastward movement prescribed for $r_{i}$ in $\Psi_{3}$, till all the robots in that row are at alternate nodes of that row. Since there is no bound on the east, there will always be vacant node available to accommodate all the robots in that particular row at the alternate nodes in that row. Thus, there is no deadlock condition.
    
    \item Case 4: If $r_{i}$ is in $\Psi_{4}$, it moves to the northward alternate row node, $T_{x,y+2}$ like, $r_{4}$ in Figure ~\ref{4}(i). It is a rigid movement and $r_{i}$ does not stop in any in-between node (i.e. $T_{x,y+1}$)  before reaching its destination. This assures no possibility of deadlock.
    
    \item Case 5: When $r_{i}$ is in $\Psi_{5}$, it either moves to the alternate northward row node $T_{x,y+2}$ or the alternate southward node $T_{x,y-2}$ as $r_{5}$ in Figure ~\ref{6}(i), depending on the position of $r_{i}$. In either case, $r_{i}$ moves to a alternate row, it will always be a odd row. Also as per the dimensions of the grid distribution calculated in Algorithm ~\ref{algo1}, there will always be vacant node available to move the robots, in order to form the grid with robots on alternate nodes, thus assuring progress of the algorithm.
\end{itemize}
Thus we see, in all the cases, $r_{i}$ is not in deadlock. Also by Lemma 1 we can assure that the robots do not move back to their previous positions and does not keep repeating the same configuration in a loop. Hence, there is no possibility of $r_{i}$ to be in deadlock condition. Thus, by contradiction we proof that the algorithm is free of deadlock and consequently progress is assured.
\end{proof}

\begin{lemma}
When a robot $r_{i}$ is moving from one node to another node of the grid, along an edge of the grid, no other robot comes in its path, i.e., the movement of $r_{i}$ is collision-free.	
\end{lemma}
\begin{proof} 
In this algorithm, the robots move along the edges of the grid to reach their destination node. For existence of any obstacle robot following two situations may arise.
\begin{itemize}
\item The obstacle robot is between the source and destination nodes. This is not possible, as the robots can be placed only on the nodes and not on the edges of the grid. So, as the robots move along the edges of the grid, between the source and destination node, no other robot can exist along the path of the robot.
\item A destination node computed by one robot is also been computed by another robot. This is not possible, as the robots have a consistent ordering of their movement based on priority. So, at a time a vacant node can be computed as the target node for one robot only. Also, the semi-synchronous scheduler confirms that when a robot is moving no other robot is looking.
\end{itemize}
Thus the robots reach their destinations without collision.
\end{proof}

\begin{theorem}
A group of autonomous, homogeneous, oblivious, semi-synchronous, mobile robots can deterministically scatter uniformly in alternate nodes of a grid under unlimited visibility and agreement on direction of axes.
\end{theorem}
\begin{proof}  
The robots in this model are autonomous mobile robots as they can operate independently by executing cycles of the states "look-compute-move". They are homogeneous and cannot be distinguished from each other. Thus, in the look phase, the robots looks to identify the positions of the other robots at that point of time. In this algorithm, the robots do not need to recollect any past action or looked data from the previous cycle, i.e they are oblivious. Here the robots have semi-synchronous scheduler which ensures, that the activated robots at the same time obtain the same snapshot, and their compute and move are executed instantaneously. Since the activated robots obtain the same snapshot, their determination of the grid dimensions, north and west bounds and the computed destination for a robot are consistent. The robots have full visibility as they need to determine the dimensions and bounds of the grid distribution to be formed. Finally, we have assumed that though the robots have local coordinate system, they have agreement on the direction of the axes. This assumption ensures that all the robots, compute and identify the same robots as the north-bound $Y_{MAX}$ robot and west-bound $X_{MIN}$ robot respectively. With this underlying model, from lemma 2 we have a deadlock free algorithm that assures progress and from lemma 3, we guarantee no collision of robots, thus deterministically reaching the goal of the algorithm, i.e., robots on alternate nodes of a grid distribution.
\end{proof}

\section{Conclusion}
In this paper, we have addressed the scattering in a grid problem for autonomous oblivious mobile robots. The paper presents a distributed algorithm that assumes full visibility robots with semi-synchronous movement and agreement on both axis direction. Our proposed algorithm converges in finite time without collision. The future scope of this work would be to consider limited visibility of the robots. Also, we will try to minimize the overall distance between the robots and make the robots occupy consecutive rows but alternative columns.

\bibliographystyle{abbrv} 
\bibliography{mybib}

\end{document}